\documentclass[letterpaper,10pt]{article}
\usepackage{amsmath}
\usepackage{amssymb}
\usepackage{mathptmx}
\usepackage{graphics}
\usepackage{graphicx}
\usepackage[utf8]{inputenc}
\usepackage[T1]{fontenc}

\usepackage{amsthm}
\usepackage{url}

\theoremstyle{plain}
\newtheorem{thm}{\protect\theoremname}

\newcounter{definition}
\theoremstyle{definition}
\newtheorem{defn}[definition]{\protect\definitionname}

\newcounter{remark}
\theoremstyle{remark}
\newtheorem{rem}[remark]{\protect\remarkname}

\newcounter{proposition}
\theoremstyle{plain}
\newtheorem{prop}[proposition]{\protect\propositionname}

\theoremstyle{plain}

\providecommand{\definitionname}{Definition}
\providecommand{\propositionname}{Porposition}
\providecommand{\corollaryname}{Corollary}

\providecommand{\remarkname}{Remark}
\providecommand{\theoremname}{Theorem}

\title{
    Virtual Holonomic Constraints in Motion Planning: \\
    Revisiting Feasibility and Limitations
}

\author{
  Maksim Surov$^{1}$
  \thanks{
    $^{1}$Maksim Surov is with Department of Information Technologies and AI, 
    Sirius University of Science and Technology, 
    Sochi, Russia
    {\tt\small surov.m.o@gmail.com}
  }
}

\begin{document}
  \maketitle
  \thispagestyle{empty}
  \pagestyle{empty}

  \begin{abstract}
    This paper addresses the feasibility of virtual holonomic constraints (VHCs) in the context of motion planning for underactuated mechanical systems with a single degree of underactuation. 
    While existing literature has established a widely accepted definition of VHC,
    we argue that this definition is overly restrictive and excludes a broad class of admissible trajectories from consideration.
    To illustrate this point, we analyze a periodic motion of the Planar Vertical Take-Off and Landing (PVTOL) aircraft that satisfies all standard motion planning requirements, including orbital stabilizability. However, for this solution -- as well as for a broad class of similar ones -- there exists no VHC that satisfies the conventional definition.
    We further provide a formal proof demonstrating that the conditions imposed by this definition necessarily fail for a broad class of trajectories of mechanical systems.
    These findings call for a reconsideration of the current definition of VHCs, with the potential to significantly broaden their applicability in motion planning.
  \end{abstract}

  \section{Introduction}
    
    In this paper, we consider Euler-Lagrange systems described by the
    equation
    \begin{equation}
    \label{eq:lag-sys}
      M\left(q\right)\ddot{q}+C\left(q,\dot{q}\right)\dot{q}+G\left(q\right)=B\left(q\right)u
    \end{equation}
    where $q\in\mathbb{R}^{n}$ denotes the generalized coordinates and
    $u\in\mathbb{R}^{n-1}$ represents the control inputs. For simplicity,
    we assume the configuration space is $\mathbb{R}^{n}$, thereby avoiding
    a discussion of its topological properties. The matrix $M\left(q\right)\in\mathbb{R}^{n\times n}$
    is positive definite; $C\left(q,\dot{q}\right)\in\mathbb{R}^{n\times n}$
    is linear in $\dot{q}$; $G\left(q\right)\in\mathbb{R}^{n}$ is a
    vector; $B\left(q\right)\in\mathbb{R}^{n\times\left(n-1\right)}$
    is of rank $n-1$. All functions are assumed to be continuously differentiable.
    This system is underactuated with underactuation degree one.
    
    For system~(\ref{eq:lag-sys}), we consider the problem of motion
    planning, which involves finding smooth functions $q_{*}\left(t\right)\in C^{2}\left(\mathbb{R}\right)$
    and $u_{*}\left(t\right)\in C^{0}\left(\mathbb{R}\right)$ that satisfy
    the system dynamics. While specific applications may impose additional
    requirements on these functions, a typical criterion is the existence of stabilizing feedback
    which ensures that the planned motion can be executed on a physical system.

    The problem of motion planning has been studied in a number of publications~\cite{Shiriaev-2005-constructive-tool,Surov-2015,Freidovich-2008,Maggiore-2013,Consolini-2011}
    using the \emph{virtual holonomic constraints} (VHCs) approach.
    This method assumes that the desired trajectory $q_{*}\left(t\right)$
    satisfies a geometric constraint of the form $h\left(q_{*}\left(t\right)\right)=0$,
    where $h:\mathbb{R}^{n}\to\mathbb{R}^{n-1}$ is a smooth function
    with Jacobian $dh\left(q\right)\equiv\frac{\partial h\left(q\right)}{\partial q}$
    of rank $n-1$ (see, e.g., \cite{Maggiore-2013,Otsason-2019}). Alternatively,
    the VHC can be defined in parametric form $q_{*}\left(t\right)=\phi\left(\theta_{*}\left(t\right)\right)$,
    where $\theta$ is a scalar parameter and $\phi:\mathbb{R}\to\mathbb{R}^{n}$
    is a smooth function. As shown in~\cite{Shiriaev-2005-constructive-tool},
    the VHC framework in application to motion planning enables reduction of the original $2n$-dimensional
    dynamics~(\ref{eq:lag-sys}) to a single scalar second-order differential
    equation, known as the \emph{reduced dynamics}. A solution $\theta_{*}\left(t\right)$
    to this equation automatically defines corresponding solution
    $q_{*}\left(t\right)=\phi\left(\theta_{*}\left(t\right)\right)$ of the
    system~(\ref{eq:lag-sys}). This technique has significantly
    simplified the motion planning problem and has facilitated the solution
    of many challenging control tasks~\cite{Surov-2015,Freidovich-2008,Buss-2016}. 
    
    In~\cite{Maggiore-2013}, the authors initiated a detailed discussion
    on the definition of VHC and the corresponding necessary conditions
    on the functions $\phi\left(\cdot\right)$ and $h\left(\cdot\right)$.
    To this end, they relate the VHCs concept to the notion of \emph{controlled
    invariant manifolds}~\cite[pp.~293-294]{Isidori-1995}. For clarity
    and brevity, we present a slightly modified version of the definition, adapted from \cite[Definition 2.1]{Maggiore-2013}:
    \begin{defn}
    \label{def:VHC}
      A virtual holonomic constraint of order $n-1$ is
      a relation $h\left(q\right)=0$, where $h:\mathbb{R}^{n}\to\mathbb{R}^{n-1}$
      is smooth, $\mathrm{rank}\,dh\left(q\right)=n-1$ for all
      $q\in h^{-1}\left(0\right)$ and the set 
      \[
          \Gamma=\left\{ \left(q,\dot{q}\right):h\left(q\right)=0\quad\text{and}\quad\frac{\partial h\left(q\right)}{\partial q}\dot{q}=0\right\} 
      \]
      is \emph{controlled invariant}. That is, there exists a smooth feedback
      $u\left(q,\dot{q}\right)$ such that $\Gamma$ is positively invariant
      for the closed-loop system. The set $\Gamma$ is called the \emph{constraint
      manifold} associated with the VHC $h\left(q\right)=0$.
    \end{defn}
    Similar definitions can be found in other publications, including~\cite{Consolini-2011,Otsason-2019,Consolini-2010,Jankuloski-2012,Consolini-2018,Elobaid-2022}.
    In some of these works~\cite{Consolini-2011,Consolini-2010,Jankuloski-2012}, the term \emph{feasible} VHC is used. A common requirement across all these studies is that the entire set $\Gamma$ must be controlled invariant. Moreover, the authors emphasize that the loss of controlled invariance of $\Gamma$ implies the VHC is \emph{not feasible}~\cite[p. 85]{Jankuloski-2012}.

    By contrast, several earlier works propose a more general interpretation of VHCs -- see, for example, \cite{Shiriaev-2005-constructive-tool,Shiriaev-2008}. In particular, the definition in~\cite[p. 204]{Shiriaev-2008} does not require the controlled invariance of $\Gamma$; instead, it merely requires the existence of at least one solution along which the relation $h(q) = 0$ holds.

    The purpose of this publication is to clarify that Definition~\ref{def:VHC} characterizes only a narrow subclass of admissible VHCs and should not be regarded as a general definition.
    The main result of the paper is formulated in Theorem~\ref{thm:singularity-necessity}, which establishes that the controlled invariance of the manifold $\Gamma$ necessarily fails for a certain class of trajectories of the mechanical system. 

    We begin our exposition in Section~\ref{sec:pvtol-trajectory} with an example of a periodic trajectory of the PVTOL aircraft that satisfies typical motion planning requirements, including orbital stabilizability. For this trajectory, we demonstrate the absence of any VHC satisfying Definition~\ref{def:VHC}. However, by relaxing the requirement of controlled invariance of $\Gamma$, the VHC-based approach remains applicable for generating the trajectory and constructing a stabilizing feedback control.
    In Section~\ref{sec:non-regular-vhc}, we examine the motion planning technique proposed in~\cite{Surov-2018}, which is based on singular reduced dynamics, and show that it conflicts with the Definition~\ref{def:VHC}.
    Section~\ref{sec:pvtol-periodic-trajectories} presents a broad class of periodic trajectories for the PVTOL aircraft, all of which fall outside the scope of Definition~\ref{def:VHC}.
    Section~\ref{sec:VHC-definitions} provides a comparative analysis of various interpretations of VHCs found in the literature. Finally, Section~\ref{sec:conclusion} offers concluding remarks.

  \subsection*{Notation}
    In accordance with~\cite{Isidori-1995}, we introduce the following
    conventions:
    \begin{itemize}
      \item
        The directional derivative of a smooth function $h:\mathbb{R}^{n}\to\mathbb{R}^{k}$
        along a vector field $f\in \mathbb{R}^{n}$ is denoted by $L_{f}h\left(q\right)\equiv\frac{\partial h\left(q\right)}{\partial q}f\left(q\right)\in\mathbb{R}^{k}$.
      \item
        The higher-order directional derivatives are defined recursively as:
        $L_{f}^{l}h\left(q\right)\equiv L_{f}L_{f}^{l-1}h\left(q\right)\in\mathbb{R}^{k}$.
      \item
        The Jacobian matrix of a function $h:\mathbb{R}^{n}\to\mathbb{R}^{k}$
        is denoted by $dh\left(q\right)\equiv\frac{\partial h\left(q\right)}{\partial q}\in\mathbb{R}^{k\times n}$.
      \item 
        The column space of a matrix $A$ is denoted by $\mathrm{Im}\left[A\right]$. 
      \item
        The kernel of a matrix $A\in\mathbb{R}^{n\times m}$ is defined as: $\ker\left[ A\right]\equiv\left\{ x\in\mathbb{R}^{m}\mid Ax=0\right\}.$
    \end{itemize}

  \section{Motivating Example: Periodic Trajectory of a PVTOL Aircraft}
  \label{sec:pvtol-trajectory}

    \begin{figure}
      \begin{centering}
        \vspace{4pt}
        \includegraphics[width=8.5cm]{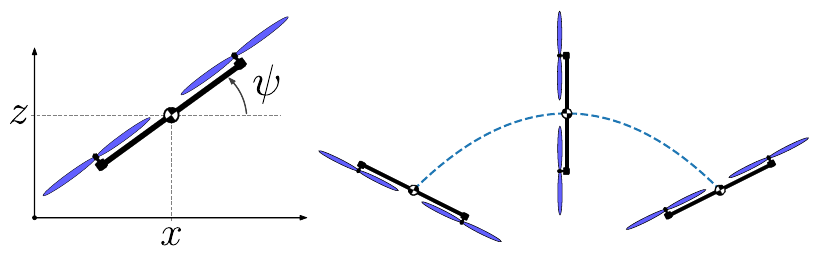}
      \par\end{centering}
      \caption{PVTOL aircraft: generalized coordinates and tic-toc maneuver schematically.}
      \label{fig:pvtol-tic-toc}
    \end{figure}

    Let us consider the PVTOL aircraft~\cite{Hauser-1992-pvtol},
    the dynamics of which are governed by the system of ordinary differential equations:
    \begin{equation}
    \label{eq:pvtol-dynamics}
      \ddot{x}=-\sin{\psi} \, u_{1},
      \quad\ddot{z}=\cos{\psi} \, u_{1} - 1,
      \quad\ddot{\psi} = u_{2},
    \end{equation}
    where $q=\left(x,z,\psi\right)^{\top}\in\mathbb{R}^{3}$ are generalized
    coordinates and $u=\left(u_{1},u_{2}\right)^{\top}\in\mathbb{R}^{2}$
    represents the control inputs, which are allowed to take both positive and negative values.
    The coordinates $x,z$ specify the position of aircraft's center of mass, while the angle $\psi$ defines its attitude, as illustrated in Fig.~\ref{fig:pvtol-tic-toc}. For simplicity of exposition, we consider a model in which the coupling factor 
    $\epsilon$ is set to zero. This assumption simplifies the subsequent expressions, yet the results can be extended to the general case with $\epsilon\ne0$ without essential modifications.
    It is straightforward to verify that the PVTOL dynamics are of the Euler-Lagrange type, with the following matrices: 
    $M = I_{3\times3}$, $C = 0_{3\times3}$, 
    \[
      G = \left(\begin{array}{c}
          0\\
          1\\
          0
      \end{array}\right)
      \quad\text{and}\quad
      B\left(q\right)=\left(\begin{array}{cc}
          -\sin\psi & 0\\
          \cos\psi & 0\\
          0 & 1
      \end{array}\right).
    \]
    In this notation, the system can be rewritten compactly as
    $$
        \ddot{q}=B\left(q\right)u-G,
    $$
    which aligns with the general structure of Euler-Lagrange systems~(\ref{eq:lag-sys}).

    It can be readily verified by direct substitution that the given pair of analytic functions
    \begin{align}
    \label{eq:ref-traj}
      q_{*}\left(t\right) & =\left(\begin{array}{c}
          \sin t\\
          -\frac{1}{2}\sin^{2}t\\
          \frac{\pi}{2}-\arctan\left(2\sin t\right)
          \end{array}\right) \in C^\omega(\mathbb{R}) \quad\text{and}\\
      u_{*}\left(t\right) & =\left(\begin{array}{c}
          \sin t\sqrt{4\sin^{2}t+1}\\
          \frac{12\sin t+2\sin3t}{\left(3-2\cos2t\right)^{2}}
          \end{array}\right)  \in C^\omega(\mathbb{R}).\nonumber 
    \end{align}
    constitutes a solution to system~(\ref{eq:pvtol-dynamics}). The trajectory $q_{*}\left(t\right)$ corresponds to the so-called \textit{tic-toc} aerobatic maneuver of the aircraft, which is schematically illustrated in Fig.~\ref{fig:pvtol-tic-toc}.

    \begin{prop}
    \label{proposition:no-vhc-for-pvtol}
      There exists no VHC satisfying Definition~\ref{def:VHC} that can be associated with trajectory~(\ref{eq:ref-traj}). In other words, there does not exist a twice continuously differentiable function $h(q) : \mathbb{R}^3 \to \mathbb{R}^2$ such that:
      \begin{enumerate}
        \item 
          $h\left(q_{*}\left(t\right)\right)\equiv0$ for all $t \in \mathbb{R}$,
        \item 
          $\mathrm{rank}\,dh\left(q\right) = 2$ for all $q \in h^{-1}\left(0\right)$,
        \item 
          the set
          \[
            \Gamma\equiv\left\{ \left(q,\dot{q}\right)\in\mathbb{R}^{6}\mid h\left(q\right)=0
             \,\,\text{and}\,\, L_{\dot{q}}h\left(q\right)=0\right\} 
          \]
          is a controlled invariant manifold.
      \end{enumerate}
    \end{prop}

    \begin{proof}
    \label{proof:proposition}
      Suppose, for the sake of contradiction, that a function $h(q)$ satisfying the conditions above exists. Then, as discussed in~\cite{Consolini-2018}, the set $\Gamma$ is a two-dimensional manifold, which represents the tangent bundle $\Gamma=T\mathcal{C}$,
      where the base space $\mathcal{C}$ is defined as 
      $$
        \mathcal{C}=\left\{ q\in\mathbb{R}^{3}\mid\exists\,t:\, q=q_{*}\left(t\right)\right\}.
      $$
      To establish the controlled invariance of $\Gamma$, we analyze the dynamics of the variable $y=h\left(q\right)$, given by 
      \begin{align}
      \label{eq:ddot_y}
        \dot{y} &= L_{\dot{q}}h\left(q\right)\nonumber \\
        \ddot{y} &= dh\left(q\right)\left(B\left(q\right)u-G\right)+L_{\dot{q}}^{2}h\left(q\right).
      \end{align}
      According to the definition of a locally controlled invariant manifold in~\cite[pp.~293-294]{Isidori-1995}, the manifold $\Gamma$ is controlled invariant if and only if, for every $\left(q, \dot{q}\right) \in \Gamma$, there exists a smooth mapping $u : \Gamma \to \mathbb{R}^2$ such that the right-hand side of equation~(\ref{eq:ddot_y}) is identically zero. This condition implies that the following system of linear algebraic equations in $u$ must be solvable for all $\left(q, \dot{q}\right) \in \Gamma$.
      \begin{equation}
      \label{eq:linsys-for-u}
        dh\left(q\right)B\left(q\right)u=dh\left(q\right)G-L_{\dot{q}}^{2}h\left(q\right).
      \end{equation}
      Here, two options are possible:
      \begin{itemize}
        \item 
          The matrix $dh\left(q\right)B\left(q\right) \in \mathbb{R}^{(n-1)\times (n-1)}$ is of full rank. In
          this case, $\Gamma$ is controlled invariant for any bounded right-hand side.
        \item 
          If the matrix $dh\left(q\right)B\left(q\right)$ is not of full rank
          at some $q_{0}\in\mathcal{C}$, then there exists a nonzero row
          vector $N$ such that $N dh\left(q\right)B\left(q\right)=0$. In this case, the necessary condition for controlled invariance becomes
          \begin{equation}
          \label{eq:dq_constr}
              Ndh\left(q_{0}\right)G-NL_{\dot{q}}^{2}h\left(q_{0}\right)=0\quad\text{for all}\quad\dot{q}\in T_{q_{0}}\mathcal{C}.
          \end{equation}
          Importantly, the first term $Ndh\left(q_{0}\right)G$ does not depend
          on $\dot{q}$, while the second term is a quadratic form in $\dot{q}$. This implies that the only way to satisfy condition~(\ref{eq:dq_constr})  is for both terms to vanish: $Ndh\left(q_{0}\right)G=0$ and $NL_{\dot{q}}^{2}h\left(q_{0}\right)=0$
          for all $\dot{q}\in T_{q_{0}}\mathcal{C}$.
      \end{itemize}
      Let us consider the point $q_{0}\in\mathcal{C}$ given by:
      $$
        q_{0}=q_{*}\left(0\right)=\left(0,0,\frac{\pi}{2}\right)^{\top}.
      $$
      At this point, the velocity vector $\dot{q}_{0}=\dot{q}_{*}\left(0\right)=\left(1,0,-2\right)^{\top}\in T_{q_{0}}\mathcal{C}$
      lies in the image of the matrix
      \begin{align*}
        B\left(q_{0}\right) & =\left(\begin{array}{cc}
        -1 & 0\\
        0 & 0\\
        0 & 1
        \end{array}\right), \quad \text{that is,} \quad
        \dot{q}_{0}\in\mathrm{Im}\left[B\left(q_{0}\right)\right] .
      \end{align*}
      The condition $h\left(q_{*}\left(t\right)\right)\equiv0$ implies
      that the rows of $dh\left(q_{0}\right)$ are orthogonal to $\dot{q}_{0}$:
      $dh\left(q_{0}\right)\dot{q}_{0}=0$. 
      Consequently, the matrix $dh\left(q_{0}\right)B\left(q_{0}\right)$ cannot have full rank, hence 
      it possesses a left annihilator $N\in\ker\left[dh\left(q_{0}\right)B\left(q_{0}\right)\right]^{\top}\setminus\left\{ 0\right\}$.
      Then, in
      order for $\Gamma$ to be controlled invariant, the condition~(\ref{eq:dq_constr})
      must be satisfied. Since the matrix $dh\left(q\right)$ consists of two linearly independent rows, the row vector $Ndh\left(q_{0}\right)\in\mathbb{R}^{1\times3}$
      is nonzero. Furthermore, the orthogonal complement of this row vector is a two-dimensional subspace that coincides with $\mathrm{Im}\left[B\left(q_{0}\right)\right]$.
      Observing that the vector $G$ does not lie in 
      $\mathrm{Im}\left[B\left(q_{0}\right)\right]$,
      we conclude that $Ndh\left(q_{0}\right)G$ is nonzero. Therefore, the equation~(\ref{eq:dq_constr}) can only hold for specific values of
      $\dot{q}$, namely $\dot{q}=\pm\dot{q}_{0}$. However, if we consider $\dot{q}=0\in T_{q_{0}}\mathcal{C}$, the second term of~(\ref{eq:dq_constr}) vanishes, causing the equality to be violated.
      This contradicts the requirement that~\eqref{eq:dq_constr} must hold for all $\dot{q} \in T_{q_{0}} \mathcal{C}$.
      Therefore, we conclude that $\Gamma$ is not a controlled invariant manifold.
    \end{proof}

    This general result can also be verified by the reader through 
    direct computation using a specific candidate for the VHC, such as
    \begin{equation}
    \label{eq:implicit-vhc-candidate}
      h\left(q\right)=\left(\begin{array}{c}
        z+\frac{1}{2}x^{2}\\
        \psi-\frac{\pi}{2}+\arctan\left(2x\right)
        \end{array}\right)
    \end{equation}
    or any other smooth function $h\left(q\right)$ that satisfies requirements 1) and 2) of Proposition~\ref{proposition:no-vhc-for-pvtol}.

    The arguments above naturally raise the question of the practical realizability of such a trajectory in a physical system, particularly in the presence of model uncertainties. One might question whether the strong requirement of passing the system through the given configuration $q_{0}$ with a precisely specific velocity is feasible, as it may appear unrealistic in practice.
    However, this requirement is a direct consequence of imposing a VHC that constrains the system to move along a prescribed direction. In real-world scenarios, slight deviations from the VHC allow the system to traverse this configuration with a range of possible velocities. Below we demonstrate, that despite the strict requirement for the velocity, the trajectory can still be stabilized with an appropriate feedback controller.

  \subsection{Stabilizability of Periodic Trajectories}
  \label{sec:pendubot-trajectory-stabilizability}
    First of all we notice, that the existence of orbitally stabilizing feedback follows directly from Theorem~4 in~\cite{Nam-1992}. This theorem states that any pair of points in a neighborhood of a closed orbit can be connected by a trajectory generated by a piecewise constant control input, provided that the Lie algebra $\mathrm{Lie}\left\{ f, g_1, g_2, \dots, g_{n-1} \right\}$ spans the entire state space \emph{at least at one point} along the orbit.
    For system~(\ref{eq:pvtol-dynamics}), the relevant vector fields are given by:
    \[
      f := \left(\begin{array}{c}
        \dot{q}\\
        -G
        \end{array}\right),\,
      g_{1} := \left(\begin{array}{c}
        0_{3\times1}\\
        \mathrm{col}_{1}B\left(q\right)
        \end{array}\right),\,
      g_{2} := \left(\begin{array}{c}
        0_{3\times1}\\
        \mathrm{col}_{2}B\left(q\right)
        \end{array}\right).
    \]
    To verify the accessibility condition, we compute the determinant of the matrix
    \[
      \Psi(q,\dot q) \equiv \left[f,g_{1},g_{2},\mathrm{ad}_{f}g_{1},\mathrm{ad}_{f}g_{2},\mathrm{ad}_{f}^{2}g_{1}\right]
    \]
    which equals to
    \[
      \det\Psi(q,\dot q) = 2\dot{\psi}\left(-\dot{\psi}\dot{x}\sin\psi+\dot{\psi}\dot{z}\cos\psi+\sin\psi\right).
    \]
    If a periodic orbit does not lie entirely within the set where $\det\Psi(q, \dot q) = 0$, then this theorem ensures the existence of a feedback law that steers the system to the reference trajectory in finite time.
    This condition is satisfied by the trajectory~(\ref{eq:ref-traj}) for all $t$ except at the endpoints $t = \pm \frac{\pi}{2}$, where $\dot \psi$ vanishes. Notably, the condition also holds at the point $q_0 = q_*(0)$, where the controlled invariance of the manifold $\Gamma$ fails.

  \subsection{Synthesis of Orbitally Stabilizing Feedback}
    While the arguments above confirm the feasibility of trajectory stabilization, practical implementation requires a more constructive synthesis procedure. 
    To this end, we adopt the transverse linearization approach~\cite{Banaszuk-1995,Surov-2020}, wherein the dynamics~(\ref{eq:pvtol-dynamics}) are expressed in new local coordinates
    $\tau \in S^1$ and $\rho \in \mathbb{R}^5$, which are diffeomorphic to $(q, \dot q)$ within a tubular neighborhood of the reference orbit.
    Although various coordinate transformations can serve this purpose (see, for example,~\cite{Surov-2020}), we aim to show that the VHC~(\ref{eq:implicit-vhc-candidate}) is also suitable, even though it does not satisfy Definition~\ref{def:VHC}. We define the new coordinates as
    \begin{align*}
      & \tau:=\mathrm{atan2}\left(x,\dot{x}\right),\quad\rho_{1,2}:=h\left(q\right),\quad\rho_{3,4}:=L_{\dot{q}}h\left(q\right),\\
      & \rho_{5}:=\left(x-x_{*}\left(\tau\right)\right)\sin\tau+\left(\dot{x}-\dot{x}_{*}\left(\tau\right)\right)\cos\tau
    \end{align*}
    along with the new control input 
    \[
      w:=u-u_{*}\left(\tau\right).
    \]
    The variables $\rho\equiv\left(\rho_{1},\dots,\rho_{5}\right)^{\top}$
    serve as transverse coordinates, meaning they vanish when the phase space point lies on the reference orbit.
    The coordinate $\tau \in [-\pi, \pi)$ represents
    projection of the phase space point onto the reference trajectory.
    It is straightforward to verify that the transformation $\left(q,\dot{q}\right)\mapsto\left(\tau,\rho\right)$
    is diffeomorphic within a tubular neighborhood of the orbit. This implies the existence of inverse transformation $\left(q,\dot{q}\right)=\mu\left(\rho,\tau\right)$
    allowing the dynamics~(\ref{eq:pvtol-dynamics}) to be expressed as
    \begin{equation}
    \label{eq:transverse-dynamics-1}
      \frac{\partial\mu\left(\rho,\tau\right)}{\partial\left(\rho,\tau\right)}\left(\begin{array}{c}
      \dot{\rho}\\
      \dot{\tau}
      \end{array}\right)=\left(\begin{array}{c}
      \dot{q}\\
      B\left(q\right)u_{*}\left(\tau\right)-G+B\left(q\right)w
      \end{array}\right).
    \end{equation}
    Next, we substitute the transformation $\left(q,\dot{q}\right)=\mu\left(\rho,\tau\right)$
    into the right-hand side of~(\ref{eq:transverse-dynamics-1})
    and solve for $\dot{\rho}$ and $\dot{\tau}$. Since the transformation is diffeomorphic, the Jacobian matrix is invertible, ensuring the validity of this step. Moreover, we observe that $\dot \tau > 0$ within a tubular neighborhood and bounded input $w$, which justifies -- following~\cite{Banaszuk-1995} -- dividing $\dot{\rho}$ by $\dot{\tau}$ to obtain the transverse dynamics in the form of a general nonlinear system:
    \begin{equation}
    \label{eq:transverse-dynamics-2}
      \frac{d\rho}{d\tau}=f\left(\rho,\tau,w\right).
    \end{equation}
    Linearizing~(\ref{eq:transverse-dynamics-2}) with respect to $\rho$ and $w$ yields:
    \[
      \frac{d\rho}{d\tau}=A\left(\tau\right)\rho+B\left(\tau\right)w+o\left(\tau,\rho,w\right),
    \]
    where the function $o\left(\tau,\rho,w\right)$ collects all bilinear, quadratic and higher-order terms in $\rho$ and $w$; the matrix functions
    \begin{align*}
      A\left(\tau\right) &:= \left(\frac{\partial f\left(\rho,\tau,w\right)}{\partial\rho}\right)_{\rho=0,w=0}\quad\text{and}\quad \\
      B\left(\tau\right) &:= \left(\frac{\partial f\left(\rho,\tau,w\right)}{\partial w}\right)_{\rho=0,w=0}
    \end{align*}
    are periodic with period $2\pi$.
    Based on the arguments in~\cite[Theorem~1]{Yakubovich-1986}, the controllability of the LTV system 
    \begin{equation}
      \frac{d\bar{\rho}}{d\tau}=A\left(\tau\right)\bar{\rho}+B\left(\tau\right)w\label{eq:linearized-transverse-dynamics}
    \end{equation}
    over one period implies the existence of an exponentially stabilizing feedback of the form $w\left(\tau,\bar\rho\right)=K\left(\tau\right)\bar\rho$ with a $2\pi$-periodic matrix of feedback gains $K\left(\tau\right):\mathbb{R}\to\mathbb{R}^{2\times5}$.
    Moreover, as shown in~\cite{Leonov-2006}, the exponential stability of the transversally linearized system guarantees the asymptotic orbital stability of the reference trajectory for the original nonlinear system.
    Following these arguments, we propose the feedback control law
    \begin{equation}
    \label{eq:nonlinear-control}
      u(q,\dot q) = u_{*}\left(\tau(q,\dot q)\right) + K\left(\tau(q,\dot q)\right) \rho(q, \dot q)
    \end{equation}
    to achieve orbital stabilization of the reference trajectory~(\ref{eq:ref-traj}).

  \subsection{Numerical Analysis}
    To validate our claims, we conducted a numerical analysis of the linearized transverse dynamics~(\ref{eq:linearized-transverse-dynamics}).
    The controllability Gramian~\cite[Proposition~5.2]{Kalman-1960-optimal-control}
    over one period has the following eigenvalues:
    \[
      \mathrm{eigvals} \, W\left(0,2\pi\right): \quad \left\{ 744.,\,70.7,\,15.3,\,5.16,\,0.0537\right\},
    \]
    confirming the controllability of the LTV system. The feedback gain matrix $K\left(\tau\right)$ was computed using the LQR~\cite{Gusev-2010}, with weighting matrices $R=I_{2\times2}$ and $Q=I_{5\times5}$. 
    The eigenvalues of the monodromy matrix $F_{2\pi}$ for the closed-loop system are:
    \begin{align*}
      \mathrm{eigvals} \, F_{2\pi}: 
       \left\{ 1.38,\,5.27\pm0.758i,\,0.268\pm5.48i\right\} \times 10^{-3},
    \end{align*}
    all of which lie inside the unit circle, ensuring exponential stability.
    To demonstrate the effectiveness of the proposed control law~(\ref{eq:nonlinear-control}),
    we performed computer simulations of the closed-loop nonlinear system with the initial state $q_{1}=\left(0.1,-0.5,0.0\right)^{\top}$, $\dot{q}_{1}=0_{3}$,
    and an integration time step of $10$ ms. The simulation results are presented in Fig.~\ref{fig:pvtol-sim}. As observed, the system state converges to the desired orbit, and all signals remain smooth and bounded. The behavior of transverse coordinates $\rho$ and control input $u$ is presented in Fig.~\ref{fig:pvtol-transverse}. 

    The analysis demonstrates the feasibility of the trajectory~(\ref{eq:ref-traj}). Based on this result, we conclude that the VHC~(\ref{eq:implicit-vhc-candidate}) is also feasible, as it holds along this trajectory. Moreover, we demonstrate in the following section, there exists a family of trajectories along which the relation~(\ref{eq:implicit-vhc-candidate}) also holds.

    \begin{figure}
      \begin{centering}
        \includegraphics[width=7.5cm]{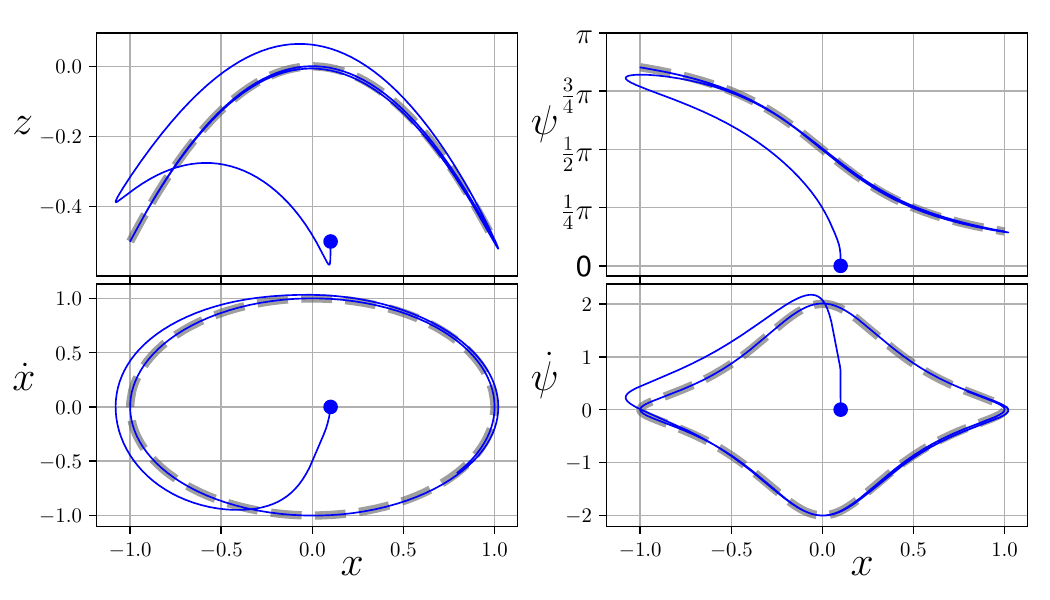}
      \par\end{centering}
      \caption{Closed loop system simulation results.}\label{fig:pvtol-sim}
    \end{figure}
    \begin{figure}
      \begin{centering}
        \vspace{4pt}
        \includegraphics[width=7.5cm]{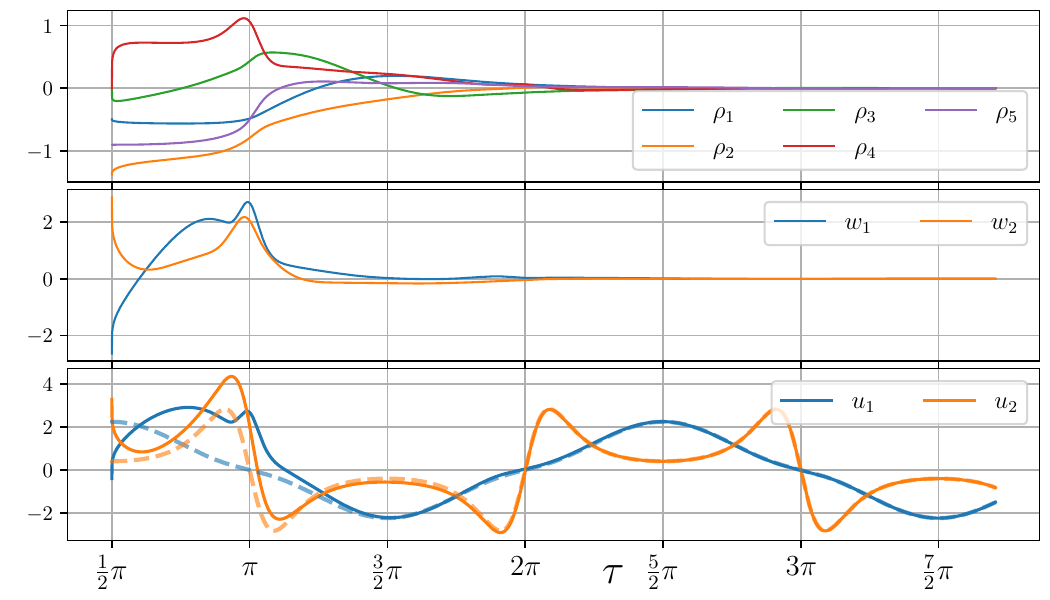}
      \par\end{centering}
      \caption{Evolution of transverse coordinates and control inputs versus trajectory projection $\tau$ in the closed loop system.}\label{fig:pvtol-transverse}
    \end{figure}

  \section{Singular Reduced Dynamics}
  \label{sec:non-regular-vhc}
    We aim to show that the trajectory considered above is not an isolated exception, but rather a representative example within a broad class of admissible solutions of mechanical systems, all of which lie out of the scope of Definition~\ref{def:VHC}. To this end, we briefly recall the method of motion planning based on VHCs~\cite{Shiriaev-2005-constructive-tool,Shiriaev-2006-periodic-planning}.
    As known, according to~\cite[Proposition~2]{Shiriaev-2005-constructive-tool},
    the dynamics of system~(\ref{eq:lag-sys}) under an imposed VHC in parametric form~$q=\phi\left(\theta\right)$
    are governed by the scalar differential equation
    \begin{equation}
    \label{eq:alpha-beta-gamma}
      \alpha\left(\theta\right)\ddot{\theta}+\beta\left(\theta\right)\dot{\theta}^{2}+\gamma\left(\theta\right)=0,
    \end{equation}
    where the scalar coefficients $\alpha\left(\theta\right),\beta\left(\theta\right)$
    and $\gamma\left(\theta\right)$ are defined as
    \begin{align*}
      \alpha\left(\theta\right) & =B_{\perp}\left(\phi\left(\theta\right)\right)M\left(\phi\left(\theta\right)\right)\phi'\left(\theta\right),\\
      \beta\left(\theta\right) & =B_{\perp}\left(\phi\left(\theta\right)\right)M\left(\phi\left(\theta\right)\right)\phi''\left(\theta\right)\\
      & \quad+B_{\perp}\left(\phi\left(\theta\right)\right)C\left(\phi\left(\theta\right),\phi'\left(\theta\right)\right)\phi'\left(\theta\right),\\
      \gamma\left(\theta\right) & =B_{\perp}\left(\phi\left(\theta\right)\right)G\left(\phi\left(\theta\right)\right)
    \end{align*}
    and $B_{\perp}\left(q\right)\in\mathbb{R}^{1\times n}\setminus\left\{ 0\right\} $
    is a left annihilator of $B\left(q\right)$.
    For the PVTOL aircraft considered in the previous section, we had used the VHC
    \begin{equation}
    \label{eq:pvtol-vhc}
      \phi\left(\theta\right)=\left(\theta,-\frac{1}{2}\theta^{2},\frac{\pi}{2}-\arctan2\theta\right)^{\top}
    \end{equation}
    and matrix $B_{\perp}\left(q\right)=\left(\cos\psi,\sin\psi,0\right)$,
    which result in the reduced dynamics:
    \begin{equation}
    \label{eq:pvtol-reduced-dynamics}
      \theta\ddot{\theta}-\dot{\theta}^{2}+1=0.
    \end{equation}
    A reader can verify through direct substitution that the piecewise function
    \begin{equation}
    \label{eq:piecewise-solution-pvtol}
      \theta_{*}\left(t\right)=\begin{cases}
          \theta_{1}\sin\frac{t}{\theta_{1}} & t\in\left[\theta_{1}\pi,0\right)\\
          \theta_{2}\sin\frac{t}{\theta_{2}} & t\in\left[0,\theta_{2}\pi\right]
        \end{cases}
    \end{equation}
    with parameters $\theta_{1}<0<\theta_{2}$, 
    is twice continuously differentiable on the interval $t\in\left[\theta_{1}\pi,\theta_{2}\pi\right]$
    and satisfies equation~(\ref{eq:pvtol-reduced-dynamics}).
    By selecting $\theta_{1}=-1$ and $\theta_{2}=1$,
    we obtain the solution $\theta_{*}\left(t\right)=\sin t$
    and consequently, the trajectory $q_{*}\left(t\right) = \phi(\theta_*(t))$ presented in~(\ref{eq:ref-traj}).
    
    As observed, equation~(\ref{eq:pvtol-reduced-dynamics}) exhibits a singularity
    at $\theta=0$, thereby violating requirements of the Picard-Lindelöf
    existence theorem. However, as shown in~\cite[Theorem~1]{Surov-2018},
    smooth solutions to the reduced dynamics with isolated singularities may still exist, 
    provided that certain additional conditions are satisfied by the coefficients
    $\alpha\left(\theta\right),\beta\left(\theta\right)$
    and $\gamma\left(\theta\right)$.
    We now recall the formulation of this existence
    \begin{thm}
    \label{thm:singular-solutions-existence}
      Let us consider the differential
      equation~(\ref{eq:alpha-beta-gamma}) with the coefficients 
      $\alpha(\theta), \beta(\theta)$ and $\gamma(\theta)$ defined on an open nontrivial interval $\theta\in I\subset\mathbb{R}$.
      Suppose the following conditions hold:
      \begin{itemize}
        \item
          the coefficients $\alpha\left(\theta\right),\beta\left(\theta\right)$
          and $\gamma\left(\theta\right)$ are smooth functions: $\alpha\left(\theta\right)\in C^{3}\left(I\right)$
          and $\beta\left(\theta\right),\gamma\left(\theta\right)\in C^{2}\left(I\right)$;
        \item
          the function $\alpha\left(\theta\right)$ has a unique zero on the
          interval: $\exists!\,\theta_{s}\in I:\alpha\left(\theta_{s}\right)=0$;
        \item
          the following inequalities hold true:
          \[
            \alpha'\left(\theta_{s}\right)>0,\quad\gamma\left(I\right)>0,\quad\frac{\beta\left(\theta_{s}\right)}{\alpha'\left(\theta_{s}\right)}<-\frac{1}{2}.
          \]
      \end{itemize}
      Then, for any real values $\theta_{1},\theta_{2}\in I$, $\theta_{1}<\theta_{s}<\theta_{2}$,
      and $\dot{\theta}_{1}\geq0$, $\dot{\theta}_{2}\geq0$ there exists
      the smooth function $\theta_{*}\left(t\right)\in C^{2}$ defined on
      the interval $\left[t_{1},t_{2}\right]$ satisfying the equation~(\ref{eq:alpha-beta-gamma})
      and the boundary conditions 
      \begin{align}
        & \theta_{*}\left(t_{1}\right)=\theta_{1},\quad\dot{\theta}_{*}\left(t_{1}\right)=\dot{\theta}_{1},\label{eq:piecewise-solution-general}\\
        \exists t_{s}\in\left(t_{1},\infty\right):\quad & \theta_{*}\left(t_{s}\right)=\theta_{s},\quad\dot{\theta}_{*}\left(t_{s}\right)=\sqrt{-\gamma\left(\theta_{s}\right)/\beta\left(\theta_{s}\right)},\nonumber \\
        \exists t_{2}\in\left(t_{s},\infty\right):\quad & \theta_{*}\left(t_{2}\right)=\theta_{2},\quad\dot{\theta}_{*}\left(t_{2}\right)=\dot{\theta}_{2}.\nonumber 
      \end{align}
    \end{thm}

    \begin{rem}
    \label{rem:periodic-solitions}
      Due to the symmetry of equation~(\ref{eq:alpha-beta-gamma})
      with respect to time inversion $t\mapsto-t$, the function $\theta_{*}\left(-t\right)$
      is also a solution to this equation.
      The reversed in time trajectory
      passes through the phase-space point with $\theta=\theta_{s}$ and $\dot{\theta}=-\sqrt{-\gamma\left(\theta_{s}\right)/\beta\left(\theta_{s}\right)}$.
      If we select $\dot{\theta}_{1}=\dot{\theta}_{2}=0$, then by concatenating the
      two functions $\theta_{*}\left(t\right)$ and $\theta_{*}\left(2t_{2}-t\right)$,
      we obtain a periodic solution with period $2t_{2}-2t_{1}$.
    \end{rem}

    It is easy to see that the reduced dynamics~(\ref{eq:pvtol-reduced-dynamics})
    satisfy conditions of Theorem~\ref{thm:singular-solutions-existence}
    for $I=\mathbb{R}$ and $\theta_{s}=0$. The solution~(\ref{eq:piecewise-solution-pvtol})
    corresponds to (\ref{eq:piecewise-solution-general}) if we take $t_{1}=\frac{\pi\theta_{1}}{2}$,
    $t_{2}=\frac{\pi\theta_{2}}{2}$ and $t_{s}=0$.

    Importantly, the singularity presented in Theorem~\ref{thm:singular-solutions-existence}
    is not removable and cannot be eliminated through a coordinate change or different representations of the VHC. To demonstrate this, consider the following
    \begin{thm}
    \label{thm:singularity-necessity} 
      Let $q_{*}\left(t\right)$ be a smooth trajectory to the system~(\ref{eq:lag-sys}). Suppose that, at some time $t_{s}$, the trajectory satisfies 
      \begin{align}
      \label{eq:singularity-condition}
        & \, \, \dot{q}_{*}\left(t_{s}\right)\ne0\quad\text{and}\nonumber \\
        & \, \, B_{\perp}\left(q_{*}\left(t_{s}\right)\right)M\left(q_{*}\left(t_{s}\right)\right)\dot{q}_{*}\left(t_{s}\right)=0,
      \end{align}
      where $q_{s}\equiv q_{*}\left(t_{s}\right)$ and $B_{\perp}\left(q\right)\in\mathbb{R}^{1\times n}\setminus\left\{ 0\right\}$  is a left annihilator of $B\left(q\right)$. Then, for any smooth parametrization of the form $q_{*}\left(t\right)=\phi\left(\theta_{*}\left(t\right)\right)$, the associated reduced dynamics necessarily satisfy
      \[
        \alpha \left(\theta_{s}\right) = 0, \quad \text{where} \quad
        \theta_{s}\equiv\theta_{*}\left(t_{s}\right).
      \]
      Furthermore, if $G\left(q_{s}\right) \notin \mathrm{Im}\left[B\left(q_{s}\right)\right]$, then 
      \begin{itemize}
        \item the reduced dynamics exhibit a non-removable singularity at $\theta_{s}$;
        \item there exists no function $h\left(q\right)$ satisfying Definition~\ref{def:VHC}, such that $h\left(q_{*}\left(t\right)\right) \equiv 0$.
      \end{itemize}
    \end{thm}

    \begin{proof}
      We substitute the expressions 
      \[
        q_{*}\left(t_{s}\right)=\phi\left(\theta_{*}\left(t_{s}\right)\right) \quad\text{and}\quad\dot{q}_{*}\left(t_{s}\right)=\phi'\left(\theta_{*}\left(t_{s}\right)\right)\dot{\theta}_{*}\left(t_{s}\right)
      \]
      into condition~(\ref{eq:singularity-condition}). This yields 
      \[
        \underbrace{B_{\perp}\left(\phi\left(\theta_{*}\left(t_{s}\right)\right)\right)M\left(\phi\left(\theta_{*}\left(t_{s}\right)\right)\right)\phi'\left(\theta_{*}\left(t_{s}\right)\right)}_{=\alpha\left(\theta_{s}\right)}\dot{\theta}_{*}\left(t_{s}\right)=0.
      \]
      Assuming that $\phi\left(\theta\right)$ is smooth and $\phi'\left(\theta_{*}\left(t_{s}\right)\right)\dot{\theta}_{*}\left(t_{s}\right)\ne0$,
      it follows that $\dot{\theta}_{*}\left(t_{s}\right)\ne0$ and therefore
      $\alpha\left(\theta_{s}\right)=0$. Now, since $G\left(q_{s}\right) \notin \mathrm{Im}\left[B\left(q_{s}\right)\right]$,
      we conclude that $\gamma\left(\theta_{s}\right) = B_\perp(q_s) G(q_s) \ne 0$. This implies
      that the right hand side of the equation 
      \[
        \ddot{\theta}=-\frac{\beta\left(\theta\right)}{\alpha\left(\theta\right)}\dot{\theta}^{2}-\frac{\gamma\left(\theta\right)}{\alpha\left(\theta\right)}
      \]
      is not Lipschitz continuous in a neighborhood of $\theta_{s}$. Hence
      the singularity at $\theta_{s}$ cannot be eliminated.

      Assume, for the sake of contradiction, that there exists a function $h(q)$ satisfying Definition~\ref{def:VHC}. Then the associated manifold
      \[
        \Gamma\equiv\left\{ \left(q,\dot{q}\right)\in\mathbb{R}^{2n}\mid h\left(q\right)=0\quad \text{and} \quad L_{\dot{q}}h\left(q\right)=0\right\}
      \]
      is controlled invariant. This implies that the following system of linear algebraic equations in $u$ must be solvable:
      \begin{align*}
        & dh\left(q\right)M^{-1}\left(q\right)\left(B\left(q\right)u-C\left(q,\dot{q}\right)\dot{q}-G\left(q\right)\right)+L_{\dot{q}}^{2}h\equiv0,
      \end{align*}
      at every point of $\Gamma$, in particular for all $\dot q \in T_{q_s} \mathcal{C}$. 
      At the point $q_s$ the matrix 
      \begin{equation}
      \label{eq:controlled-invariance-condition}
        dh\left(q_{s}\right)M^{-1}\left(q_{s}\right)B\left(q_{s}\right) \in \mathbb{R}^{(n-1)\times(n-1)}
      \end{equation}
      becomes singular and hence admits a nontrivial left annihilator $N\in\mathbb{R}^{1\times(n - 1)}\setminus\left\{ 0\right\}$. This fact follows from the properties:
      $B_{\perp}\left(q_{s}\right)M\left(q_{s}\right)\dot{q}_{s}=0$, $\dot{q}_{s}\ne0$ and $L_{\dot{q}} h = 0$. Since the expression~(\ref{eq:controlled-invariance-condition}) contains terms which are bilinear in $\dot{q}$ and terms which do not depend on $\dot{q}$, the only way for the system to remain solvable is if the following two conditions hold:
      \begin{align}
        N dh\left(q_{s}\right)M^{-1}\left(q_{s}\right)G\left(q_{s}\right) &= 0 \label{eq:g_projection_is_zero} \\
        NL_{\dot{q}}^{2}h\left(q_{s}\right)-Ndh\left(q_{s}\right)M^{-1}\left(q_{s}\right)C\left(q_{s},\dot{q}\right)\dot{q} &= 0 \label{eq:force_balance}
      \end{align}
      for all $\dot{q}\in T_{q_{s}}\mathcal{C}$. Here, the row vector $N dh\left(q_{s}\right)M^{-1}\left(q_{s}\right) \in \mathbb{R}^{1\times n}$ is nonzero. Its orthogonal complement coincides with the $(n-1)$-dimensional linear space $\mathrm{Im}\left[B\left(q_{s}\right)\right]$. Consequently, the condition~(\ref{eq:g_projection_is_zero}) implies that $G(q_s) \in \mathrm{Im}\left[B\left(q_{s}\right)\right]$, which directly contradicts the assumption of Theorem~\ref{thm:singularity-necessity}. This contradiction immediately implies that there is no function $h(q)$ which defines a VHC for the considering trajectory and satisfies Definition~\ref{def:VHC}.
    \end{proof}

    Theorem~\ref{thm:singularity-necessity} establishes that Definition~\ref{def:VHC} cannot be satisfied for a broad class of trajectories, thereby precluding their generation via the VHC approach.
    Specifically, consider a candidate VHC of the form $q = \phi(\theta)$, and suppose that a corresponding solution $\theta_*(t)$ to the reduced dynamics is obtained using Theorem~\ref{thm:singular-solutions-existence}. Then the resulting trajectory $q_*(t) = \phi(\theta_*(t))$ necessarily satisfies the conditions of Theorem~\ref{thm:singularity-necessity}, implying that no valid VHC exists in the sense of Definition~\ref{def:VHC}.
    However, there are no fundamental reasons to disregard trajectories like $q_*(t)$ or to consider the associated constraint $\phi(\theta)$ invalid, as the trajectory satisfies all typical requirements of the motion planning problem, as previously shown. 
    In this context, Definition~\ref{def:VHC} plays a restrictive role by excluding such trajectories from further consideration. Consequently, the definition becomes counterproductive in the context of motion planning. In the following, we illustrate -- using the example of the PVTOL system -- a family of practically relevant trajectories, yet are excluded by the definition.

  \section{Example: PVTOL Aircraft Trajectories}
  \label{sec:pvtol-periodic-trajectories}
    In the previous section, we demonstrated that the VHC~\eqref{eq:pvtol-vhc} leads to reduced dynamics with an isolated singularity at $\theta_s = 0$, corresponding to the configuration $x_s = 0, z_s = 0$ and $\psi_s = \frac{\pi}{2}$. Notably, it is possible to define a VHC that imposes a singularity at an arbitrary configuration $q_{s} \equiv \left(x_{s}, z_{s}, \psi_{s}\right)^{\top}$ as follows:
    \begin{equation}
    \label{eq:pvtol-vhc-poly}
      \phi\left(\theta\right)=q_{s}+B\left(q_{s}\right)\left(\begin{array}{c}
        k_{1}\\
        k_{2}
      \end{array}\right)\theta+\frac{k_{3}}{2}B_{\perp}^{\top}\left(q_{s}\right)\theta^{2},
    \end{equation}
    where $k_{1}, k_{2}, k_{3} \in \mathbb{R}$ are constant parameters. The corresponding reduced dynamics take the form
    \begin{align}
    \label{eq:alpha-beta-gamma-poly-vhc}
      & \alpha\left(\theta\right)\ddot{\theta}+\beta\left(\theta\right)\dot{\theta}^{2}+\gamma\left(\theta\right)=0,\quad\text{with}\\
      & \alpha\left(\theta\right)=k_{1}\sin\left(k_{2}\theta\right)+k_{3}\theta\cos\left(k_{2}\theta\right), \nonumber  \\
      & \beta\left(\theta\right)=k_{3}\cos\left(k_{2}\theta\right),\quad\gamma\left(\theta\right)=\sin\left(\psi_{s}+k_{2}\theta\right)\nonumber 
    \end{align}
    and exhibit a singularity at $\theta_s = 0$, which corresponds to the configuration $q_s = \phi(0)$. If, for a given $q_s$, one can choose parameters $k_1,k_2,k_3$ such that equation~(\ref{eq:alpha-beta-gamma-poly-vhc}) satisfies the conditions of Theorem~\ref{thm:singular-solutions-existence}, then a family of periodic trajectories exists in a neighborhood of $q_s$.
    A straightforward analysis shows that these conditions are solvable for any $x_s \in \mathbb{R}$, $y_s \in \mathbb{R}$, and $\psi_s \in (0, \pi) \cup (\pi, 2\pi)$ and hence existence of periodic trajectories near these configurations.
    Taking into account the discussion in Section~\ref{sec:pendubot-trajectory-stabilizability}, we conclude that most of these trajectories are orbitally stabilizable.

    We performed a numerical analysis of two trajectories from this set corresponding to small oscillations near $\psi_s = \frac{\pi}{4}$ and $\psi_s = \frac{\pi}{2}$. Simulation results of feedback-controlled periodic motions for them are available at~\url{https://youtube.com/shorts/EM2qfht1D74}.
    It is important to emphasize that, for all of these trajectories, no function $h(q)$ satisfying Definition~\ref{def:VHC} exists, thereby highlighting the limitations of the definition.

  \section{Discussion on VHC Definitions}
  \label{sec:VHC-definitions}
    To the best of our knowledge, Definition~\ref{def:VHC} was first formally introduced in~\cite{Jankuloski-2012}, although the requirement for the controlled invariance of $\Gamma$ appeared in earlier works such as~\cite{Consolini-2010,Consolini-2011}. 
    This definition represents a shift toward a more narrow interpretation of VHCs
    compared to earlier studies, for example, the broader formulation presented in~\cite[p.~204]{Shiriaev-2008}. In that work, the concept of a VHC is based on the existence of a feedback control law $u$ that ensures the invariance of the relation $q = \phi(\theta)$ along at least one solution $q_*(t)$ of the system~(\ref{eq:lag-sys}).
    A similar interpretation is found in~\cite[pp. 1165-1167]{Shiriaev-2005-constructive-tool}, where the authors require the existence of a static feedback control such that the relation $q = \phi(\theta)$ is invariant along solutions of the closed-loop system. This broader viewpoint is consistent with the results established in Theorems~\ref{thm:singular-solutions-existence} and~\ref{thm:singularity-necessity}.

    By narrowing the interpretation of VHCs, the authors of~\cite{Jankuloski-2012} did not provide a formal justification for excluding the earlier, broader formulations. Moreover, they did not address whether the conditions introduced in Definition~\ref{def:VHC} are necessary. Subsequent works adopting this definition \cite{Maggiore-2013,Consolini-2018,Otsason-2019} have not clarified this issue either. 

    In our view, the definition provided in~\cite{Shiriaev-2008} is sufficiently general and well-suited for most trajectory planning tasks, with the possible exception of cases involving non-smooth virtual constraints. Meanwhile, the more general formulation presented in~\cite{Celikovsky-2015,Celikovsky-2017} requires only the linear independence of the differentials $dh_i(q)$, for $i = 1, \dots, n-1$, and does not address the feasibility of the constraint relation $h(q) = 0$. This approach is likewise appropriate in the context of motion planning.

  \section{Concluding Remarks}
  \label{sec:conclusion}
    This work has critically examined the commonly used definition of VHCs based on controlled invariance. We have shown that this definition represents a significant narrowing of the broader conceptual framework introduced in earlier works on VHCs.
    In particular, we have formally proven that the controlled invariance of the manifold $\Gamma$ must violate for a broad class of feasible trajectories. 
    Consequently, the feasibility of a VHC is not solely determined by the existence of a controlled invariant manifold and, therefore, should not be treated as a fundamental criterion in the general definition of VHCs.

    Moreover, the requirement of controlled invariance is at odds with motion planning techniques presented in~\cite{Surov-2018,Surov-2025}. In this context, narrowing the concept of VHCs plays a counterproductive role, as it restricts the development of more general and practically useful motion planning approaches.

    It would be more appropriate to regard Definition~\ref{def:VHC} as describing a subclass of VHCs in which the controlled invariance of $\Gamma$ is intrinsic. By narrowing the general interpretation of VHCs, the authors introduced terminology that became misleading in the broader context of motion planning, potentially hindering future developments in the field.

\end{document}